\documentclass{article} %
\usepackage{iclr2020_conference,times}

\newcommand{\eg}{{\it e.g.}}
\newcommand{\ie}{{\it i.e.}}
\newcommand{\etc}{{\it etc.}}

\newcommand{\BA}{\begin{array}}
\newcommand{\EA}{\end{array}}

\newcommand{\BIT}{\begin{itemize}}
\newcommand{\EIT}{\end{itemize}}

\newcommand{\argmin}{\mathop{\rm argmin}}
\newcommand{\argmax}{\mathop{\rm argmax}}

\usepackage{times}  %
\usepackage{helvet}  %
\usepackage{courier}  %
\usepackage{url}  %
\usepackage{graphicx}  %
\usepackage{mathtools}

\usepackage{hyperref}       %
\usepackage{amsfonts}       %
\usepackage{amsmath}

\usepackage{dsfont}
\usepackage{algorithm}
\usepackage{algpseudocode}
\usepackage{caption}
\usepackage{subcaption}%
\usepackage{amsthm}
\usepackage{color}
\usepackage{morefloats}
\usepackage{epstopdf}
\usepackage{xspace}

\usepackage{caption}
\usepackage{changepage}
\usepackage{enumitem}
\usepackage{graphicx}

\newcommand{\bsuite}{\texttt{bsuite}}
\newcommand{\bsuiteversion}{\texttt{bsuite2019}}
\newcommand{\bsuitegithub}{\href{http://github.com/deepmind/bsuite}{\texttt{github.com/deepmind/bsuite}}}

\newcommand{\bsuitetitle}[1]{
\rule{\linewidth}{4pt}
\vspace{-5mm}

\section{\hfil \LARGE \normalfont 
\bsuite\ report: #1
\vspace{2mm} \hfil
\vspace{-3mm}
}

\rule{\linewidth}{1pt}
}

\newcommand{\bsuiteabstract}{
{\small
\begin{adjustwidth}{1.5cm}{1.5cm}
The \textit{Behaviour Suite for Reinforcement Learning}, or \bsuite\ for short, is a collection of carefully-designed experiments that investigate core capabilities of a reinforcement learning (RL) agent.
The aim of the \bsuite\ project is to collect clear, informative and scalable problems that capture key issues in the design of efficient and general learning algorithms and study agent behaviour through their performance on these shared benchmarks.
This report provides a snapshot of agent performance on \bsuiteversion, obtained by running the experiments from \bsuitegithub\ \cite{osband2019bsuite}.
\end{adjustwidth}
}
}

\newcommand{\Ac}{\mathcal{A}}

\newcommand{\Sc}{\mathcal{S}}
\newcommand{\Rc}{\mathcal{R}}
\newcommand{\Pc}{\mathcal{P}}
\newcommand{\Opt}{\mathcal{O}}
\newcommand{\Hc}{\mathcal{H}}
\newcommand{\Fc}{\mathcal{F}}

\newcommand{\Mc}{\mathcal{M}}

\DeclarePairedDelimiterX{\infdivx}[2]{(}{)}{%
  #1\;\delimsize|\delimsize|\;#2%
}
\newcommand{\kld}[2]{\ensuremath{D_{KL}\infdivx{#1}{#2}}\xspace}

\newcommand{\Exp}{\mathbb{E}}
\newcommand{\Nat}{\mathbb{N}}
\newcommand{\Prob}{\mathbb{P}}
\newcommand{\Real}{\mathbb{R}}

\DeclareMathOperator*{\dprime}{{\prime \prime}}

\newtheorem*{theorem*}{Theorem}
\newtheorem{theorem}{Theorem}

\newtheorem{problem}{Problem}

\title{Making Sense of Reinforcement Learning\\and Probabilistic Inference}

\author{
Brendan O'Donoghue\thanks{These authors contributed equally to this work.}\, \thanks{DeepMind, London, UK, \texttt{\{bodonoghue,iosband,cdi\}@google.com}}
\And Ian Osband \footnotemark[1]\, \footnotemark[2]
\And Catalin Ionescu \footnotemark[2]
}

\iclrfinalcopy
\begin{document}
\maketitle

\begin{abstract}
Reinforcement learning (RL) combines a control problem with statistical
estimation: The system dynamics are not known to the agent, but can be learned
through experience.  A recent line of research casts `RL as inference' and
suggests a particular framework to generalize the RL problem as probabilistic
inference.  Our paper surfaces a key shortcoming in that approach, and clarifies
the sense in which RL can be coherently cast as an inference problem.  In
particular, an RL agent must consider the effects of its actions upon future
rewards and observations: The exploration-exploitation tradeoff.  In all but the
most simple settings, the resulting inference is computationally intractable so
that practical RL algorithms must resort to approximation.  We demonstrate that
the popular `RL as inference' approximation can perform poorly in even very
basic problems. However, we show that with a small modification the
framework does yield algorithms that can provably perform well, and we
show that the resulting algorithm is equivalent to the recently proposed
K-learning, which we further connect with Thompson sampling.
\end{abstract}

\section{Introduction}
\label{sec:intro}

Probabilistic inference is a procedure of making sense of uncertain data using
Bayes' rule.  The optimal control problem is to take actions in a known system
in order to maximize the cumulative rewards through time.  Probabilistic
graphical models (PGMs) offer a coherent and flexible language to specify causal
relationships, for which a rich literature of learning and inference techniques
have developed \citep{koller2009probabilistic}.  Although control dynamics might
also be encoded as a PGM, the relationship between action planning and
probabilistic inference is not immediately clear.  For inference, it is
typically enough to specify the system and pose the question, and the objectives
for learning emerge automatically.  In control, the system and objectives are
known, but the question of how to approach a solution may remain
extremely complex \citep{bertsekas2005dynamic}.

Perhaps surprisingly, there is a deep sense in which inference and control can
represent a dual view of the same problem.  This relationship is most clearly
stated in the case of linear quadratic systems, where the Ricatti equations
relate the optimal control policy in terms of the system dynamics
\citep{welch1995introduction}.  In fact, this connection extends to a wide range
of systems, where control tasks can be related to a dual inference problem
through rewards as exponentiated probabilities in a distinct, but coupled, PGM
\citep{todorov2007linearly, todorov2008general}.  A great benefit of this
connection is that it can allow the tools of inference to make progress in
control problems, and vice-versa.  In both cases the connections provide new
insights, inspire new algorithms and enrich our understanding
\citep{toussaint2006probabilistic, ziebart2008maximum, kappen2012optimal}.

Reinforcement learning (RL) is the problem of learning to control an unknown
system \citep{sutton2018reinforcement}.  Like the control setting, an RL agent
should take actions to maximize its cumulative rewards through time.  Like the
inference problem, the agent is initially uncertain of the system dynamics, but
can learn through the transitions it observes.  This leads to a
fundamental tradeoff: The agent may be able to improve its understanding through
exploring poorly-understood states and actions, but it may be able to attain
higher immediate reward through exploiting its existing knowledge
\citep{kearns2002near}.  In many ways, RL combines control and inference into a
general framework for decision making under uncertainty.  Although there has
been ongoing research in this area for many decades, there has been a recent
explosion of interest as RL techniques have made high-profile breakthroughs in
grand challenges of artificial intelligence research
\citep{mnih-atari-2013,silver2016mastering}.

A popular line of research has sought to cast `RL as inference', mirroring the
dual relationship for control in known systems.  This approach is most clearly
stated in the tutorial and review of \cite{levine2018rlasinf}, and provides a
key reference for research in this field.  It suggests that a
\textit{generalization} of the RL problem can be cast as probabilistic inference
through inference over exponentiated rewards, in a continuation of previous work
in optimal control \citep{todorov2009efficient}.  This perspective promises
several benefits: A probabilistic perspective on rewards, the ability to apply
powerful inference algorithms to solve RL problems and a natural exploration
strategy.  In this paper we will outline an important way in which this
perspective is incomplete. This shortcoming ultimately results in algorithms
that can perform poorly in even very simple decision problems.
Importantly, these are not simply technical issues that show up in some edge
cases, but fundamental failures of this approach that arise in even the
most simple decision problems.

In this paper we revisit an alternative framing of `RL as inference'.  In fact,
we show that the \textit{original} RL problem was already an inference problem
all along.\footnote{Note that, unlike control, connecting RL with inference will
not involve a separate `dual' problem.} Importantly, this inference problem
includes inference over the agent's future actions and observations.  Of course,
this perspective is not new, and has long been known as simply the Bayes-optimal
solution, see, \eg, \citet{ghavamzadeh2015bayesian}.  The problem is that, due
to the exponential lookahead, this inference problem is fundamentally
intractable for all but the simplest problems \citep{gittins1979bandit}.  For
this reason, RL research focuses on computationally efficient approaches that
maintain a level of statistical efficiency \citep{furmston2010variational,
osband2017deep}.

We provide a review of the RL problem in Section \ref{sec:rl_problem}, together
with a simple and coherent framing of RL as probabilistic inference. In Section
\ref{sec:approx_bayes_opt} we present three approximations to the intractable
Bayes-optimal policy. We begin with the celebrated Thompson sampling algorithm,
then we review the popular `RL as inference' framing, as presented by
\cite{levine2018rlasinf}, and highlight a clear and simple shortcoming in
this approach. Finally, we review K-learning \citep{klearning}, which we
re-interpret as a modification to the RL as inference framework that provides a
principled approach to the statistical inference problem, as well as a
presenting a relationship with Thompson sampling. In Section
\ref{sec:computation} we present computational studies that support our claims.

\section{Reinforcement learning}
\label{sec:rl_problem}

We consider the problem of an agent taking actions in an unknown environment in
order to maximize cumulative rewards through time.  For simplicity, this paper
will model the environment as a finite horizon, discrete Markov Decision Process
(MDP) $M = (\Sc, \Ac, \Rc, \Pc, H, \rho)$.\footnote{This choice is for clarity;
continuous, infinite horizon, or partially-observed environments do not alter
our narrative.} Here $\Sc=\{1,..,S\}$ is the state space, $\Ac=\{1,..,A\}$ is
the action space and each episode is of fixed length $H \in \Nat$.  Each episode
$\ell \in \Nat$ begins with state $s_0 \sim \rho$ then for timesteps
$h=0,..,H-1$ the agent selects action $a_h$, observes transition $s_{h+1}$ with
probability $\Pc(s_{h+1}, s_h, a_h) \in [0,1]$ and receives reward $r_{h+1} \sim
\Rc(s_h, a_h)$, where we denote by $\mu(s_h, a_h) = \Exp r_{h+1}$ the mean
reward.  We define a policy $\pi$ to be a mapping from $\Sc$ to probability
distributions over $\Ac$ and write $\Pi$ for the space of all policies.  For any
timestep $t = (\ell, h)$, we define $\Fc_t=(s^0_0, a^0_0, r^0_1, ..,
s^\ell_{h-1}, a^\ell_{h-1}, r^\ell_{h})$ to be the sequence of observations made
before time $t$.  An RL algorithm maps histories to policies $\pi_t = {\rm
alg}(\Sc, \Ac, \Fc_t)$.

Our goal in the design of RL algorithms is to obtain good performance
(cumulative rewards) for an unknown $M \in \Mc$, where $\Mc$ is some
\textit{family} of possible environments.  Note that this is a different problem
from typical `optimal control', that seeks to optimize performance for one
particular known MDP $M$; although you might still fruitfully apply an RL
\textit{algorithm} to solve problems of that type.  For any environment $M$ and
any policy $\pi$ we can define the action-value function,
\begin{equation}
  \label{eq:value}
  Q_h^{M, \pi}(s, a) = \Exp_{\pi, M}\left[\sum_{j=h+1}^H r_j \mid s_h =s, a_h=a \right].
\end{equation}
Where the expectation in \eqref{eq:value} is taken with respect to the action
selection $a_j$ for $j>h$ from the policy $\pi$ and evolution of the fixed MDP
$M$.  We define the value function $V^{M, \pi}_h(s) = \Exp_{\alpha \sim \pi}
Q_h^{M, \pi}(s,\alpha)$ and write $Q_h^{M, \star}(s,a) = \max_{\pi \in
\Pi}Q^{M,\pi}_h(s,a)$ for the optimal Q-values over policies, and the optimal
value function is given by $V_h^{M, \star}(s) = \max_a Q_h^{M, \star}(s,a)$.

In order to compare algorithm performance across different environments, it is
natural to normalize in terms of the \textit{regret}, or shortfall in cumulative
rewards relative to the optimal value, 
\begin{equation}
  \label{eq:regret}
  {\rm Regret}(M, {\rm alg}, L) =
    \Exp_{M, {\rm alg}} \left[ \sum_{\ell=1}^L  \left(V^{M,\star}_0(s^\ell_0)
    -  \sum_{h=1}^H r^\ell_h\right)  \right].
\end{equation}
This quantity depends on the unknown MDP $M$, which is fixed from the start and
kept the same throughout, but the expectations are taken with respect to the
dynamics of $M$ and the learning algorithm ${\rm alg}$.  For any particular MDP
$M$, the optimal regret of zero can be attained by the non-learning algorithm
${\rm alg}_M$ that returns the optimal policy for $M$.

In order to assess the quality of a reinforcement learning algorithm, which is
designed to work across some \textit{family} of $M \in \Mc$, we need some method
to condense performance over a set to a single number.  There are two main
approaches to this:
\begin{eqnarray}
  \label{eq:bayes_regret}
  {\rm BayesRegret}(\phi, {\rm alg}, L) = \Exp_{M \sim \phi} {\rm Regret}(M, {\rm alg}, L), \\
  \label{eq:minimax_freq}
  {\rm WorstCaseRegret}(\Mc, {\rm alg}, L) = \max_{M \in \Mc} {\rm Regret}(M, {\rm alg}, L),
\end{eqnarray}
where $\phi$ is a prior over the family $\Mc$.  These differing objectives are
often framed as Bayesian (average-case) \eqref{eq:bayes_regret} and frequentist
(worst-case) \eqref{eq:minimax_freq} RL \footnote{Some frequentist
results are high-probability bounds on the worst case rather than true
worst-case bounds, but this distinction is not important for our purposes}.
Although these two settings are typically studied in isolation, it should be
clear that they are intimately related through the choice of $\Mc$ and $\phi$.
Our next section will investigate what it would mean to `solve' the RL problem.
Importantly, we show that both frequentist and Bayesian perspectives already
amount to a problem in probabilistic inference, without the need for additional
re-interpretation.

\subsection{Solving the RL problem through probabilistic inference}
\label{sec:solving_rl}

If you want to `solve' the RL problem, then formally the objective is clear:
find the RL algorithm that minimizes your chosen objective,
(\ref{eq:bayes_regret}) or (\ref{eq:minimax_freq}).  To anchor our discussion,
we introduce a simple decision problem designed to highlight some
key aspects of reinforcement learning.  We will revisit this problem setting as
we discuss approximations to the optimal policy.
\begin{problem}[One unknown action]
\label{problem:one_unknown}
Fix $N \in \Nat \ge 3, \epsilon > 0$ and define $\Mc_{N, \epsilon} =
\{M_{N,\epsilon}^+, M_{N,\epsilon}^-\}$.  Both $M^+$ and $M^-$ share $\Sc=\{1\},
H=1$ and $\Ac=\{1, .., N\}$; they only differ through their rewards:
\begin{eqnarray*}
\label{eq:rewards}
    \Rc^+(1) = 1, \quad \Rc^+(2) = +2, \quad \Rc^+(a) = 1 - \epsilon
        \ \text{ for } \ a=3,..,N, \\
    \Rc^-(1) = 1, \quad \Rc^-(2) = -2, \quad \Rc^-(a) = 1 - \epsilon
        \ \text{ for } \ a=3,..,N.
\end{eqnarray*}
Where $\Rc(a) = x \in \Real$ is a shorthand for deterministic reward of $x$ when
choosing action $a$.
\end{problem}

Problem \ref{problem:one_unknown} is extremely simple, it involves no
generalization and no long-term consequences: It is an independent bandit
problem with only one unknown action.  For \textit{known} $M^+, M^-$ the optimal
policy is trivial: Choose $a_t=2$ in $M^+$ and $a_t=1$ in $M^-$ for all $t$.  An
RL agent faced with \emph{unknown} $M \in \Mc$ should attempt to optimize the RL
objectives \eqref{eq:bayes_regret} or \eqref{eq:minimax_freq}.  Unusually, and
only because Problem \ref{problem:one_unknown} is so simple, we can actually
compute the \textit{optimal} solutions to both in terms of $L$ (the total number
of episodes) and $\phi = (p^+, p^-)$ where $p^+ = \Prob(M = M^+)$, the
probability of being in $M^+$.

For $L > 3$ an optimal \textit{minimax} (minimizing the worst-case regret) RL
algorithm is to first choose $a_0=2$ and  observe $r_1$.  If $r_1 = 2$ then you
know you are in $M^+$ so pick $a_t=2$ for all $t = 1,2..$, for ${\rm Regret}(L)
= 0$.  If $r_1 = -2$ then you know you are in $M^-$ so pick $a_t=1$ for all $t =
1,2..$, for ${\rm Regret}(L) = 3$.  The worst-case regret of this algorithm is
$3$, which cannot be bested by any algorithm.

Actually, the same RL algorithm is also \textit{Bayes}-optimal for any $\phi =
(p^+, p^-)$ provided $p^+ L > 3$.  This relationship is not a coincidence.  All
admissible solutions to the worst-case problem \eqref{eq:minimax_freq} are given
by solutions to the average-case \eqref{eq:bayes_regret} for some `worst-case'
prior $\tilde{\phi}$ \citep{wald1950statistical}.  As such, for ease of
exposition, our discussion will focus on the Bayesian (or average-case) setting.
However, readers should understand that the same arguments apply to the
worst-case objective.

In Problem \ref{problem:one_unknown}, the key probabilistic inference the agent
must consider is the effects of it own \textit{actions} upon the future rewards,
\ie, whether it has chosen action $2$.  Slightly more generally, where actions
are independent and episode length $H=1$, the optimal RL algorithm can be
computed via Gittins indices, but these problems are very much the exception
\citep{gittins1979bandit}.  In problems with generalization or long-term
consequences, computing the Bayes-optimal solution is computationally
intractable.  One example of an algorithm that converges to Bayes-optimal
solution in the limit of infinite computation is given by Bayes-adaptive
Monte-Carlo Planning \citep{guez2012efficient}.  The problem is that, even for
very simple problems, the lookahead tree of interactions between actions,
observations and algorithmic updates grows exponentially in the search depth
\citep{strehl2006pac}.  Worse still, direct computational approximations to the
Bayes-optimal solution can fail exponentially badly should they fall short of
the required computation \citep{munos2014bandits}.  As a result, research in
reinforcement learning amounts to trying to find computationally tractable
approximations to the Bayes-optimal policy that maintain some degree of
statistical efficiency.

\section{Approximations for computational and statistical efficiency}
\label{sec:approx_bayes_opt}

The exponential explosion of future actions and observations means solving
for the Bayes-optimal solution is computationally intractable.  To counter this,
most computationally efficient approaches to RL simplify the problem at time $t$
to only consider inference over the data $\Fc_t$ that has been gathered prior to
time $t$.  The most common family of these algorithms are `certainty equivalent'
(under an identity utility): They take a point estimate for their best guess of
the environment $\hat{M}$, and try to optimize their control given these
estimates $V^{\hat{M}, \star}$.  Typically, these algorithms are used in
conjunction with some dithering scheme for random action selection (\eg,
epsilon-greedy), to mitigate premature and suboptimal convergence
\citep{watkins1989learning}.  However, since these algorithms do not prioritize
their exploration, they may take exponentially long to find the optimal policy
\citep{osband2014generalization}.

In order for an RL algorithm to be statistically efficient, it must consider the
value of information.  To do this, an agent must first maintain some notion of
epistemic uncertainty, so that it can direct its exploration towards states and
actions that it does not understand well \citep{o2017uncertainty}.  Here again,
probabilistic inference finds a natural home in RL: We should build up posterior
estimates for the unknown problem parameters, and use this \textit{distribution}
to drive efficient exploration.\footnote{For the purposes of this paper, we will
focus on \textit{optimistic} approaches to exploration, although more
sophisticated information-seeking approaches merit investigation in future work
\citep{russo2014ids}.}

\subsection{Thompson sampling}
\label{sec:TS}

One of the oldest heuristics for balancing exploration with exploitation is
given by Thompson sampling, or probability matching
\citep{thompson1933likelihood}.  Each episode, Thompson sampling (TS) randomly
selects a policy according to the probability it is the optimal policy,
conditioned upon the data seen prior to that episode.  Thompson sampling is a
simple and effective method that successfully balances exploration with
exploitation \citep{russo2018tutorial}.

Implementing Thompson sampling amounts to an inference problem at each episode.
For each $s,a,h$ define the binary random variable $\Opt_h(s,a)$ where
$\Opt_h(s,a)=1$ denotes the event that action $a$ is optimal for state $s$ in
timestep $h$.\footnote{For the problem definition in Section
\ref{sec:rl_problem} there is always a deterministic optimal policy for $M$.}
The TS policy for episode $\ell$ is thus given by the inference problem,
\begin{equation}
\label{eq:TS}
    \pi^{\rm TS} \sim \Prob(\Opt \mid \Fc_{\ell}),
\end{equation}
where $\Prob(\Opt \mid \Fc_{\ell})$ is the \emph{joint} probability over all the
binary optimality variables (hereafter we shall suppress the dependence on
$\Fc_\ell$).
To understand how Thompson sampling guides exploration let us consider its
performance in Problem \ref{problem:one_unknown} when implemented with a uniform
prior $\phi=(\frac{1}{2}, \frac{1}{2})$.  In the first timestep the agent
samples $M_0 \sim \phi$.  If it samples $M^+$ it will choose action $a_0=2$ and
learn the true system dynamics, choosing the optimal arm thereafter.  If it
samples $M^-$ it will choose action $a_0=1$ and repeat the identical decision in
the next timestep.  Note that this procedure achieves BayesRegret 2 according
to $\phi$, but \textit{also} worst-case regret $3$, which matches the optimal
minimax performance despite its uniform prior.

Recent interest in TS was kindled by strong empirical performance in bandit
tasks \citep{chapelle2011empirical}.  Following work has shown that this
algorithm satisfies strong Bayesian regret bounds close to the known lower
bounds for MDPs, under certain assumptions
\citep{osband2016posterior,osband2016lower}.  However, although much simpler
than the Bayes-optimal solution, the inference problem in \eqref{eq:TS} can
still be prohibitively expensive.  Table \ref{table:TS} describes one approach
to performing the sampling required in \eqref{eq:TS} implicitly, by maintaining
an explicit model over MDP parameters.  This algorithm can be computationally
intractable as the MDP becomes large and so attempts to scale Thompson sampling
to complex systems have focused on \textit{approximate} posterior samples via
randomized value functions, but it is not yet clear under which settings these
approximations should be expected to perform well \citep{osband2017deep}.  As we
look for practical, scalable approaches to posterior inference one promising
(and popular) approach is known commonly as `RL as inference'.

\begin{table}[H]
\begin{center}
\caption{Model-based Thompson sampling.}
\label{table:TS}
\begin{tabular}{ c | c }
Before episode $\ell$ & Sample $M_\ell = (\Sc, \Ac, \Rc^\ell, \Pc^\ell, H, \rho) \sim \phi \mid \Fc_\ell$ \\
Bellman equation &
\parbox{8cm}{
\[
\begingroup
\renewcommand{\arraystretch}{1.7} %
\begin{array}{c}
Q_h^\ell(s , a) = \mu^\ell(s, a)+ \sum_{s'}\Pc^\ell(s', s, a) V^\ell_{h+1}(s')\\
V^\ell_h(s) = \max_a Q_h^\ell(s,a)
\end{array}
\endgroup
\]
} \\
Policy & $\pi_h^{\rm TS}(s,a) \in \argmax Q_h^\ell(s, a)$
\end{tabular}
\end{center}
\end{table}

\subsection{The `RL as inference' framework and its limitations}
\label{sec:rl_as_inference}

The computational challenges of Thompson sampling suggest an approximate
algorithm that replaces \eqref{eq:TS} with a parametric distribution suitable
for expedient computation.  It is possible to view the algorithms of the `RL as
inference' approach in this light \citep{rawlik2013stochastic,
todorov2009efficient, toussaint2009robot, deisenroth2013survey,
fellows2019virel}; see \cite{levine2018rlasinf} for a recent survey.  These
algorithms choose to model
the probability of optimality according to,
\begin{equation}
\label{eq:delta_function}
    \tilde \Prob(\Opt_h(s,a) | \tau_h(s,a) ) \propto \exp\left(\sum_{(s^\prime,a^\prime) \in \tau_h(s,a)} \beta \Exp^\ell \mu(s^\prime, a^\prime)\right).
\end{equation}
for some $\beta > 0$, where $\tau_h(s,a)$ is a trajectory (a sequence of
state-action pairs) starting from $(s,a)$ at timestep $h$, and where $\Exp^\ell$
denotes the expectation under the posterior at episode $\ell$.
With this potential in place one can perform Bayesian inference over the
unobserved `optimality' variables, obtaining posteriors over the policy or other
variables of interest.  This presentation of the RL as inference framework is
slightly closer to the one in \citet[\S 2.4.2.2]{deisenroth2013survey} than to
\citet{levine2018rlasinf}, but ultimately it produces the same family of
algorithms. We provide such a derivation in the appendix for completeness.

Applying inference procedures to \eqref{eq:delta_function} leads naturally to RL
algorithms with some `soft' Bellman updates, and added entropy regularization.
We describe the general structure of these algorithms in Table \ref{t-sq}.
These algorithmic connections can help reveal connections to policy gradient,
actor-critic, and maximum entropy RL methods \citep{mnih2016asynchronous,
o2016pgq, haarnoja17, haarnoja2018soft, eysenbach2018diversity}.  The problem is
that this resultant `posterior' derived using \eqref{eq:delta_function} does not
generally bear any close relationship to the agent's epistemic probability that
$(s,a,h)$ is optimal.

\begin{table}[H]
\begin{center}
\caption{Soft Q-learning.}
\label{t-sq}
\begin{tabular}{ c | c }
Bellman equation &
\parbox{8cm}{
\[
\begingroup
\renewcommand{\arraystretch}{1.7} %
\begin{array}{c}
\tilde Q_h(s, a) =  \Exp^\ell \mu(s, a)+ \sum_{s'}\Exp^\ell \Pc(s', s, a) \tilde V_{h+1}(s')\\
\tilde V_h(s) = \beta^{-1} \log \sum_{a} \exp\beta \tilde Q_h(s, a)
\end{array}
\endgroup
\]
} \\
Policy & $\pi_h^{\rm SQ}(s, a) \propto \exp \beta \tilde Q_h(s, a)$
\end{tabular}
\end{center}
\end{table}

To understand how `RL as inference' guides decision making, let us consider its
performance in Problem \ref{problem:one_unknown}.  Practical implementations of
`RL as inference' estimate $\Exp^\ell \mu$ through observations.  For $N$ large,
and without prior guidance, the agent is then extremely unlikely to select
action $a_t=2$ and so resolve its epistemic uncertainty.  Even for an informed
prior $\phi=(\frac{1}{2}, \frac{1}{2})$ action selection according to the
exploration strategy of Boltzmann dithering is unlikely to sample
action $2$ for which $\Exp^\ell \mu(2)=0$
\citep{levine2018rlasinf,cesa2017boltzmann}.  This is because the $N-1$
`distractor' actions with $\Exp^\ell \mu \ge 1-\epsilon$ are much more probable
under the Boltzmann policy.

This problem is the same problem that afflicts most dithering approaches to
exploration.  `RL as inference' as a framework does not incorporate an agents
epistemic uncertainty, and so can lead to poor policies for even simple
problems.  While \eqref{eq:delta_function} allows the construction of a dual
`posterior distribution', this distribution does not generally bear any relation
to the typical posterior an agent should compute conditioned upon the data it has
gathered, \eg, equation \eqref{eq:TS}.  Despite this shortcoming RL as inference
has inspired many interesting and novel techniques, as well as delivered
algorithms with good performance on problems where exploration is not the
bottleneck \citep{eysenbach2018diversity}.  However, due to the use of language
about `optimality' and `posterior inference' \etc, it may come as a surprise to
some that this framework does not truly tackle the Bayesian RL problem.  Indeed,
algorithms using `RL as inference' can perform very poorly on problems where
accurate uncertainty quantification is crucial to performance.  We hope that
this paper sheds some light on the topic.

\subsection{Making sense of `RL as Inference' via K-learning}
\label{sec:K-learning}

In this section we suggest a subtle alteration to the `RL as inference'
framework that develops a coherent notion of optimality.  The K-learning
algorithm was originally introduced through a risk-seeking exponential utility
\citep{klearning}.  In this paper we re-derive this algorithm as a principled
approximate inference procedure with clear connections to Thompson sampling, and
we highlight its similarities to the `RL as inference' framework.  We believe that
this may offer a road towards combining the respective strengths of Thompson
sampling and the `RL as inference' frameworks.  First, consider the following
approximate conditional optimality probability at $(s,a,h)$:
\begin{equation}
\label{e-kl-approx}
\tilde \Prob( \Opt_h(s, a) | Q^{M, \star}_h(s,a) ) \propto \exp \beta Q^{M, \star}_h(s,a),
\end{equation}
for some $\beta > 0$,
and note that this is conditioned on the random variable $Q^{M, \star}_h(s,a)$.
We can marginalize over possible Q-values yielding
\begin{equation}
\label{e-pol-cgf}
\tilde \Prob(\Opt_h(s, a)) = \int \tilde \Prob(\Opt_h(s, a) | Q^{M, \star}_h(s,a)) d\Prob(Q^{M, \star}_h(s,a)) \propto \exp G^{Q}_h(s, a, \beta),
\end{equation}
where $G^{Q}_h(s, a, \cdot)$ denotes the cumulant generating function of the random
variable $Q^{M, \star}_h(s,a)$ \citep{kendall1946advanced}.  Clearly K-learning
and the `RL as inference' framework are similar, since equations
\eqref{eq:delta_function} and \eqref{e-kl-approx} are closedly linked, but there
is a crucial difference.  Notice that the integral performed in
\eqref{e-pol-cgf} is with respect to the \emph{posterior} over $Q^{M,
\star}_h(s,a)$, which includes the epistemic uncertainty explicitly.

\begin{table}[H]
\begin{center}
\caption{K-learning.}
\label{t-kl}
\begin{tabular}{ c | c }
Before episode $\ell$ &  Calculate $\beta_\ell = \beta\sqrt{\ell}$\\
Bellman equation &
\parbox{10cm}{
\[
\begingroup
\renewcommand{\arraystretch}{1.7} %
\begin{array}{c}
\displaystyle K_{h}(s,a) = \Exp^\ell \mu(s,a) + \frac{\sigma^2 \beta_\ell}{2 n^\ell(s,a)}+ \sum_{s^\prime}\Exp^\ell \Pc(s', s, a)V^{\rm K}_{h+1}(s')\\
V^{\rm K}_h(s) = \beta_\ell^{-1} \log \sum_{a} \exp \beta_\ell K_h(s,a)
\end{array}
\endgroup
\]
} \\
Policy & $\pi_h^{\rm K}(s, a) \propto \exp \beta_\ell K_h(s,a)$
\end{tabular}
\end{center}
\end{table}

Given the approximation to the posterior probability of optimality in
(\ref{e-pol-cgf}) we could sample actions from it as our policy, as done by
Thompson sampling \eqref{eq:TS}. However, that requires computation of the
cumulant generating function $G^Q_h(s,a, \beta)$, which is non-trivial.  It was
shown in \citep{klearning} that an upper bound to the cumulant generating
function could be computed by solving a particular `soft' Bellman equation.  The
resulting K-values, denoted $K_h(s,a)$ at $(s,a, h)$, are also optimistic for
the expected optimal Q-values. Specifically, for any sequence $\{\beta_\ell\}$
the following holds
\begin{equation}
\label{e-optimism}
K_{h}(s,a) \geq \beta_\ell^{-1} G^Q_h(s,a, \beta_\ell) \geq \Exp^\ell Q^{M, \star}_h(s,a).
\end{equation}
Following a Boltzmann policy over these K-values satisfies a Bayesian regret
bound which matches the current best bound for Thompson sampling up to
logarithmic factors under the same set of assumptions.  We summarize the
K-learning algorithm in Table (\ref{t-kl}), where $\beta > 0$ is a constant and
and $n^\ell(s,a)$ is the \emph{visitation count} of $(s,a)$ before episode
$\ell$, \ie, the number of times the agent has taken action $a$ at state $s$,
and $\sigma > 0$ is a constant. The uncertainty in the transition function is
incorporated into the constant $\sigma$, which is a technical detail we omit
here for clarity, see \citep{klearning} for details. In this way the agent is
given a reward signal that includes a bonus which is higher for states and
actions that the agent has visited less frequently.

Comparing Tables \ref{t-sq} and \ref{t-kl} it is clear that soft Q-learning and
K-learning share some similarities: They both solve a `soft' value function and
use Boltzmann policies.  However, the differences are important.  Firstly,
K-learning has an explicit schedule for the inverse temperature parameter
$\beta_\ell$, and secondly it adds a bonus based on visitation count to the
expected reward.  These two relatively small changes make K-learning a
principled exploration and inference strategy.

To understand how K-learning drives exploration, consider its performance on
Problem \ref{problem:one_unknown}.  Since this is a bandit problem we can
compute the cumulant generating functions for each arm and then use the policy
given by (\ref{e-pol-cgf}). For any non-trivial prior and choice of $\beta > 0$
the cumulant generating function is optimistic for arm $2$ which results in the
policy selecting arm $2$ more frequently, thereby resolving its epistemic
uncertainty. As $\beta \rightarrow \infty$ K-learning converges to the policy of
pulling arm $2$ deterministically.  This is in contrast to soft Q-learning where
arm $2$ is exponentially \textit{unlikely} to be selected as the exploration
parameter $\beta$ grows.

\subsubsection{Connections between K-learning and Thompson sampling}
\label{sec:K_thompson}

Since K-learning can be viewed as approximating the posterior probability of
optimality of each action it is natural to ask how close an approximation it is.
A natural way to measure this similarity is the Kullback–Leibler (KL) divergence
between the distributions,
\[
\kld{\Prob(\Opt_h(s))}{\pi_h^{\rm K}(s)} = \sum_{a} \Prob(\Opt_h(s,a)) \log (\Prob(\Opt_h(s,a)) / \pi_h^{\rm K}(s,a)),
\]
where we are using the notation $\Opt_h(s) = \Opt_h(s, \cdot)$ and $\pi_h^{\rm
K}(s) = \pi_h^{\rm K}(s, \cdot)$.  This is different to the usual notion of
distance that is taken in variational Bayesian methods, which would typically
reverse the order of the arguments in the KL divergence
\citep{blundell2015weight}.  However, in RL that `direction' is not appropriate:
a distribution minimizing $\kld{\pi_h(s)}{\Prob(\Opt_h(s))}$ may put zero
probability on regions of support of $\Prob(\Opt_h(s))$.  This means an action
with non-zero probability of being optimal might \emph{never} be taken.  On the
other hand a policy minimizing $\kld{\Prob(\Opt_h(s))}{\pi_h(s)}$ must assign a
non-zero probability to every action that has a non-zero probability of being
optimal, or incur an infinite KL divergence penalty.  With this characterization
in mind, and noting that the Thompson sampling policy satisfies $\Exp^\ell
\pi^{\rm TS}_h(s) = \Prob(\Opt_h(s))$, our next result links the policies of
K-learning to Thompson sampling.
\begin{theorem} \label{thm-kl}
The K-learning value function $V^{\rm K}$ and policy $\pi^{\rm K}$ defined in
Table \ref{t-kl} satisfy the following bound at every state $s \in \Sc$ and
$h=0, \ldots H$:
\begin{equation}
\label{e-thm-kl}
V^{\rm K}_h(s) \geq \Exp V^{M, \star}_h(s) + \beta^{-1} \kld{\Prob(\Opt_h(s))}{\pi_h^{\rm K}(s)}.
\end{equation}
\end{theorem}
We defer the proof to Appendix \ref{app:proof_thm}.  This theorem tells us that
the distance between the true probability of optimality and the K-learning
policy is bounded for any choice of $\beta < \infty$. In other words, if there
is an action that might be optimal then K-learning will eventually take that
action.

\subsection{Why is `RL as Inference' so popular?}
\label{sec:rl_ai_popular}

The sections above outline some surprising ways that the `RL as inference'
framework can drive suboptimal behaviour in even simple domains.  The question
remains, why do so many popular and effective algorithms lie within this class?
The first, and most important point, is that these algorithms can perform
extremely well in domains where efficient exploration is not a bottleneck.
Furthermore, they are often easy to implement and amenable to function
approximation \citep{peters2010relative, kober2009policy, map2018}.  Our
discussion of K-learning in Section \ref{sec:K-learning} shows that a relatively
simple fix to this problem formulation can result in a framing of RL as
inference that maintains a coherent notion of optimality.  Computational results
show that, in tabular domains, K-learning can be competitive with, or even
outperform Thompson sampling strategies, but extending these results to
large-scale domains with generalization is an open question
\citep{klearning,osband2017deep}.

The other observation is that the `RL as inference' can provide useful insights
to the structure of particular \textit{algorithms} for RL.  It is valid to note
that, under certain conditions, following policy gradient is equivalent to a
dual inference problem where the `probabilities' play the role of dummy
variables, but are not supposed to represent the probability of optimality in
the RL problem.  In this light, \citet{levine2018rlasinf} presents the inference
framework as a way to generalize a wide range of state of the art RL algorithms.
However, when taking this view, you should remember that this inference duality
is limited to certain RL algorithms, and without some modifications (e.g.
Section \ref{sec:K-learning}) this perspective is in danger of overlooking
important aspects of the RL problem.

\section{Computational experiments}
\label{sec:computation}

\subsection{One unknown action (Problem \ref{problem:one_unknown})}
\label{sec:one_unknown}

Consider the environment of Problem \ref{problem:one_unknown} with uniform prior
$\phi=(\frac{1}{2}, \frac{1}{2})$.  We fix $\epsilon = 1e-3$ and consider how
the Bayesian regret varies with $N > 3$.  Figure \ref{fig:problem_1} compares
how the regret scales for Bayes-optimal ($1.5$), Thompson sampling ($2$),
K-learning ($\le2.2$) and soft Q-learning (which grows linearly in $N$ for the
optimal $\beta \rightarrow 0$, but would typically grow exponentially for $\beta
> 0$).  This highlights that, even in a simple problem, there can be great value
in considering the value of information.

\begin{figure}[h!]
\centering
\begin{minipage}{.47\textwidth}
  \vspace{0mm}
  \centering
  \includegraphics[width=1\linewidth]{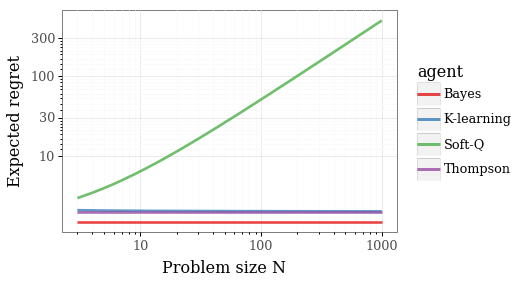}
  \vspace{-7mm}
  \captionof{figure}{Regret scaling on Problem \ref{problem:one_unknown}.\\ Soft Q-learning does not scale gracefully\\ with $N$.}
  \label{fig:problem_1}
\end{minipage}%
\vspace{2mm}
\begin{minipage}{.47\textwidth}
  \centering
  \vspace{-6mm}
  \includegraphics[width=.6\linewidth]{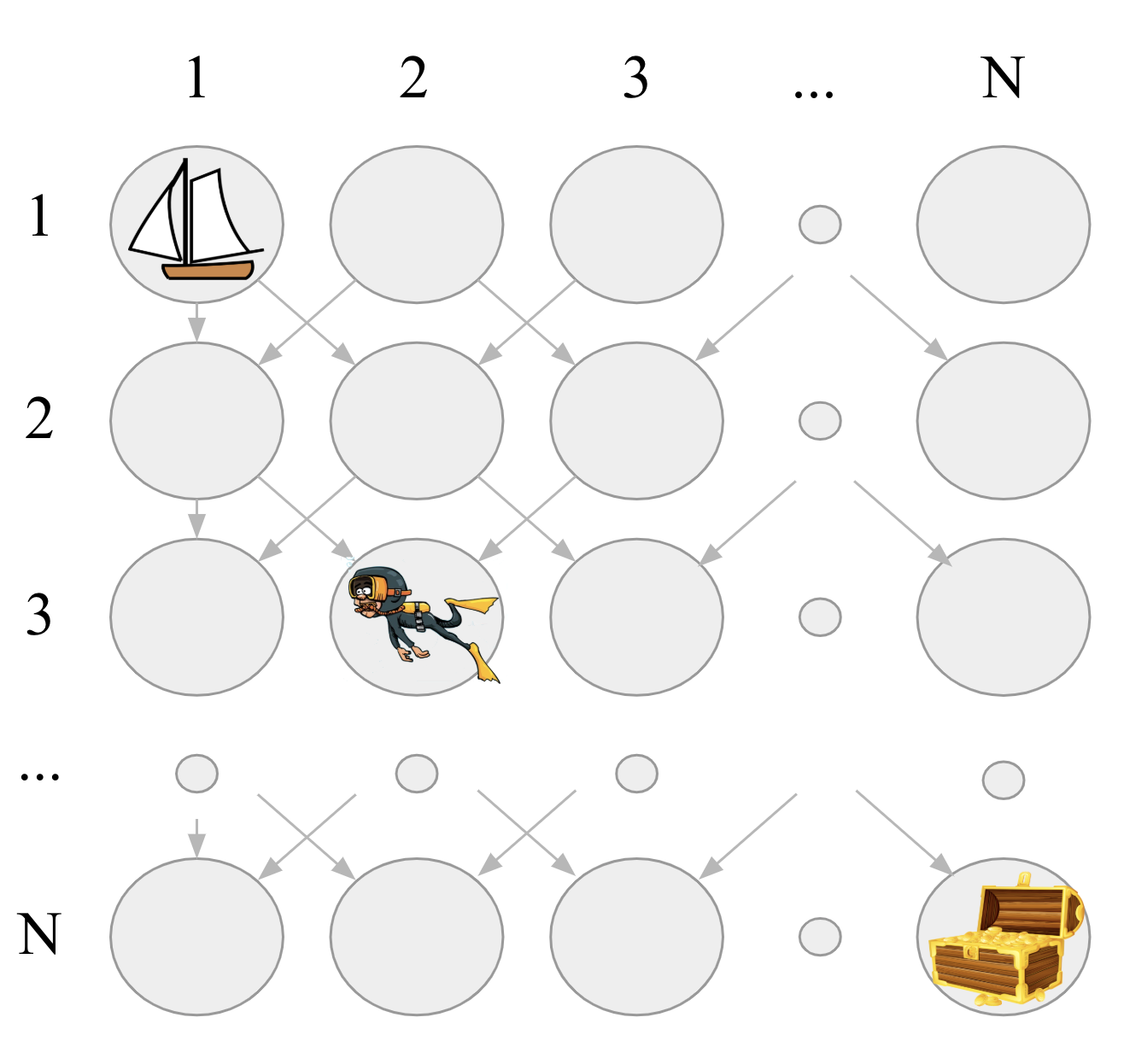}
  \vspace{-1mm}
  \captionof{figure}{DeepSea exploration: A simple\\ example where deep exploration is critical.}
  \label{fig:deepsea-pic}
\end{minipage}
\end{figure}

\subsection{`DeepSea' exploration}

Our next set of experiments considers the `DeepSea' MDPs introduced by
\citet{osband2017deep}.  At a high level this problem represents a `needle in a
haystack', designed to require efficient exploration, the complexity of which
grows with the problem size $N \in \Nat$.  DeepSea
(Figure~\ref{fig:deepsea-pic}) is a scalable variant of the `chain MDPs' popular
in exploration research \citep{jaksch2010near}.  \footnote{DeepSea figure taken
with permission from the `bsuite' \cite{osband2019bsuite}}

The agent begins each episode in the top-left state in an $N \times N$ grid.  At
each timestep the agent can move left or right one column, and falls one row.
There is a small negative reward for heading right, and zero reward for left.
There is only one rewarding state, at the bottom right cell.  The only way the
agent can receive positive reward is to choose to go right in each
timestep.  Algorithms that do not perform \textit{deep exploration} will take an
exponential number of episodes to learn the optimal policy, but those that
prioritize informative states and actions can learn much faster.

Figure \ref{fig:deep_sea} shows the `time to learn' for tabular implementations
of K-learning (Section \ref{sec:K-learning}), soft Q-learning (Section
\ref{sec:rl_as_inference}) and Thompson sampling (Section \ref{sec:TS}).  We
implement each of the algorithms with a $N(0,1)$ prior for rewards and ${\rm
Dirichlet}(1/N)$ prior for transitions.  Since these problems are small and
tabular, we can use conjugate prior updates and exact MDP planning
via value iteration.  As expected, Thompson sampling and K-learning scale
gracefully to large domains but soft Q-learning does not.

\subsection{Behaviour Suite for Reinforcement Learning}

So far our experiments have been confined to the tabular setting, but the main
focus of `RL as inference' is for scalable algorithms that work with
generalization.  In this section we show that the same insights we built in the
tabular setting extend to the setting of deep RL.  To do this we implement
variants of Deep Q-Networks with a single layer, 50-unit MLP
\citep{mnih-atari-2013}.  To adapt K-learning and Thompson sampling to this deep
RL setting we use an ensemble of size 20 with randomized prior functions to
approximate the posterior distribution over neural network Q-values
\citep{osband2018randomized} (full experimental details are included in Appendix
\ref{app:bsuite-implementation}).  We then evaluate all of the algorithms on
\texttt{bsuite}: A suite of benchmark tasks designed to highlight key issues in
RL \citep{osband2019bsuite}.

In particular, \texttt{bsuite} includes an evaluation on the DeepSea problems
but with a one-hot pixel representation of the agent position.  In Figure
\ref{fig:deep_sea_fn} we see that the results for these deep RL implementations
closely match the observed scaling for the tabular setting.  In particular, the
algorithms motivated by Thompson sampling and K-learning both scale gracefully
to large problem sizes, where soft Q-learning is unable to drive deep
exploration.  Our \texttt{bsuite} evaluation includes many more experiments that
can be fit into this paper, but we provide a link to the complete results at
\href{http://bit.ly/rl-inference-bsuite}{\texttt{bit.ly/rl-inference-bsuite}}.
In general, the results for Thompson sampling and K-learning are similar, with
soft Q-learning performing significantly worse on `exploration' tasks.  We push
a summary of these results to Appendix \ref{app:bsuite-report}.

\begin{figure}[h!]
\centering
\begin{subfigure}[t]{.46\textwidth}
  \centering
  \includegraphics[height=40mm, width=\linewidth, keepaspectratio]{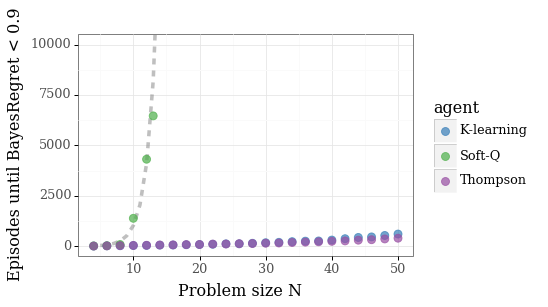}
  \caption{Tabular state representation.}
  \label{fig:deep_sea}
\end{subfigure}
\hspace{4mm}
\begin{subfigure}[t]{.46\textwidth}
  \centering
  \includegraphics[height=40mm, width=\linewidth, keepaspectratio]{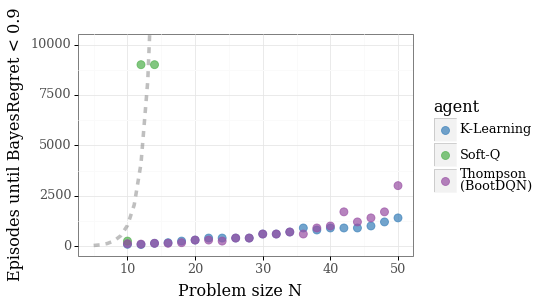}
  \caption{One-hot pixel representation into neural net.}
  \label{fig:deep_sea_fn}
\end{subfigure}
\vspace{-2mm}
\caption{Learning times for DeepSea experiments. Dashed line represents $2^N$.}
\label{fig:memory_len}
\end{figure}

\section{Conclusion}

This paper aims to make sense of reinforcement learning and probabilistic
inference.  We review the reinforcement learning problem and show that this
problem of optimal learning already combined the problems of control and
inference.  As we highlight this connection, we also clarify some potentially
confusing details in the popular `RL as inference' framework.  We show that,
since this problem formulation ignores the role of epistemic uncertainty, that
algorithms derived from that framework can perform poorly on even simple tasks.
Importantly, we also offer a way forward, to reconcile the views of RL and
inference in a way that maintains the best pieces of both.  In particular, we
show that a simple variant to the RL as inference framework (K-learning) can
incorporate uncertainty estimates to drive efficient exploration.  We support
our claims with a series of simple didactic experiments.  We leave the crucial
questions of how to scale these insights up to large complex domains for future
work.

\bibliographystyle{iclr2020_conference}
\bibliography{rl_inference}

\begin{thebibliography}{53}
\providecommand{\natexlab}[1]{#1}
\providecommand{\url}[1]{\texttt{#1}}
\expandafter\ifx\csname urlstyle\endcsname\relax
  \providecommand{\doi}[1]{doi: #1}\else
  \providecommand{\doi}{doi: \begingroup \urlstyle{rm}\Url}\fi

\bibitem[Abdolmaleki et~al.(2018)Abdolmaleki, Springenberg, Tassa, Heess, and
  Riedmiller]{map2018}
Abbas Abdolmaleki, Jost~Tobias Springenberg, Yuval Tassa, Remi Munos~Nicolas
  Heess, and Martin Riedmiller.
\newblock Maximum a posteriori policy optimisation.
\newblock In \emph{International Conference on Learning Representations
  (ICLR)}, 2018.

\bibitem[Asmussen \& Glynn(2007)Asmussen and Glynn]{asmussen2007stochastic}
S{\o}ren Asmussen and Peter~W Glynn.
\newblock \emph{Stochastic simulation: {A}lgorithms and analysis}, volume~57.
\newblock Springer Science \& Business Media, 2007.

\bibitem[Bertsekas(2005)]{bertsekas2005dynamic}
Dimitri~P Bertsekas.
\newblock \emph{Dynamic programming and optimal control}, volume~1.
\newblock Athena Scientific, 2005.

\bibitem[Blundell et~al.(2015)Blundell, Cornebise, Kavukcuoglu, and
  Wierstra]{blundell2015weight}
Charles Blundell, Julien Cornebise, Koray Kavukcuoglu, and Daan Wierstra.
\newblock Weight uncertainty in neural networks.
\newblock \emph{arXiv preprint arXiv:1505.05424}, 2015.

\bibitem[Cesa-Bianchi et~al.(2017)Cesa-Bianchi, Gentile, Neu, and
  Lugosi]{cesa2017boltzmann}
Nicol{\`o} Cesa-Bianchi, Claudio Gentile, Gergely Neu, and Gabor Lugosi.
\newblock Boltzmann exploration done right.
\newblock In \emph{Advances in Neural Information Processing Systems}, pp.\
  6287--6296, 2017.

\bibitem[Chapelle \& Li(2011)Chapelle and Li]{chapelle2011empirical}
Olivier Chapelle and Lihong Li.
\newblock An empirical evaluation of thompson sampling.
\newblock In \emph{Advances in neural information processing systems}, pp.\
  2249--2257, 2011.

\bibitem[Deisenroth et~al.(2013)Deisenroth, Neumann, Peters,
  et~al.]{deisenroth2013survey}
Marc~Peter Deisenroth, Gerhard Neumann, Jan Peters, et~al.
\newblock A survey on policy search for robotics.
\newblock \emph{Foundations and Trends{\textregistered} in Robotics},
  2\penalty0 (1--2):\penalty0 1--142, 2013.

\bibitem[Eysenbach et~al.(2018)Eysenbach, Gupta, Ibarz, and
  Levine]{eysenbach2018diversity}
Benjamin Eysenbach, Abhishek Gupta, Julian Ibarz, and Sergey Levine.
\newblock Diversity is all you need: Learning skills without a reward function.
\newblock \emph{arXiv preprint arXiv:1802.06070}, 2018.

\bibitem[Fellows et~al.(2019)Fellows, Mahajan, Rudner, and
  Whiteson]{fellows2019virel}
Matthew Fellows, Anuj Mahajan, Tim~GJ Rudner, and Shimon Whiteson.
\newblock Virel: A variational inference framework for reinforcement learning.
\newblock In \emph{Advances in Neural Information Processing Systems}, pp.\
  7120--7134, 2019.

\bibitem[Furmston \& Barber(2010)Furmston and Barber]{furmston2010variational}
Thomas Furmston and David Barber.
\newblock Variational methods for reinforcement learning.
\newblock In \emph{Proceedings of the Thirteenth International Conference on
  Artificial Intelligence and Statistics}, pp.\  241--248, 2010.

\bibitem[Ghavamzadeh et~al.(2015)Ghavamzadeh, Mannor, Pineau, and
  Tamar]{ghavamzadeh2015bayesian}
Mohammad Ghavamzadeh, Shie Mannor, Joelle Pineau, and Aviv Tamar.
\newblock Bayesian reinforcement learning: {A} survey.
\newblock \emph{Foundations and Trends{\textregistered} in Machine Learning},
  8\penalty0 (5-6):\penalty0 359--483, 2015.

\bibitem[Gittins(1979)]{gittins1979bandit}
John~C Gittins.
\newblock Bandit processes and dynamic allocation indices.
\newblock \emph{Journal of the Royal Statistical Society. Series B
  (Methodological)}, pp.\  148--177, 1979.

\bibitem[Guez et~al.(2012)Guez, Silver, and Dayan]{guez2012efficient}
Arthur Guez, David Silver, and Peter Dayan.
\newblock Efficient {B}ayes-adaptive reinforcement learning using sample-based
  search.
\newblock In \emph{Advances in Neural Information Processing Systems}, pp.\
  1025--1033, 2012.

\bibitem[Haarnoja et~al.(2017)Haarnoja, Tang, Abbeel, and Levine]{haarnoja17}
Tuomas Haarnoja, Haoran Tang, Pieter Abbeel, and Sergey Levine.
\newblock Reinforcement learning with deep energy-based policies.
\newblock In \emph{Proceedings of the 34th International Conference on Machine
  Learning ({ICML})}, 2017.

\bibitem[Haarnoja et~al.(2018)Haarnoja, Zhou, Abbeel, and
  Levine]{haarnoja2018soft}
Tuomas Haarnoja, Aurick Zhou, Pieter Abbeel, and Sergey Levine.
\newblock Soft actor-critic: {O}ff-policy maximum entropy deep reinforcement
  learning with a stochastic actor.
\newblock \emph{arXiv preprint arXiv:1801.01290}, 2018.

\bibitem[Jaksch et~al.(2010)Jaksch, Ortner, and Auer]{jaksch2010near}
Thomas Jaksch, Ronald Ortner, and Peter Auer.
\newblock Near-optimal regret bounds for reinforcement learning.
\newblock \emph{Journal of Machine Learning Research}, 11\penalty0
  (Apr):\penalty0 1563--1600, 2010.

\bibitem[Kappen et~al.(2012)Kappen, G{\'o}mez, and Opper]{kappen2012optimal}
Hilbert~J Kappen, Vicen{\c{c}} G{\'o}mez, and Manfred Opper.
\newblock Optimal control as a graphical model inference problem.
\newblock \emph{Machine learning}, 87\penalty0 (2):\penalty0 159--182, 2012.

\bibitem[Kearns \& Singh(2002)Kearns and Singh]{kearns2002near}
Michael Kearns and Satinder Singh.
\newblock Near-optimal reinforcement learning in polynomial time.
\newblock \emph{Machine Learning}, 49\penalty0 (2-3):\penalty0 209--232, 2002.

\bibitem[Kendall(1946)]{kendall1946advanced}
Maurice~George Kendall.
\newblock \emph{The advanced theory of statistics.}
\newblock Charles Griffin and Co., Ltd., London, 1946.

\bibitem[Kingma \& Ba(2014)Kingma and Ba]{kingma2014adam}
Diederik Kingma and Jimmy Ba.
\newblock Adam: {A} method for stochastic optimization.
\newblock arXiv preprint arXiv:1412.6980, 2014.

\bibitem[Kober \& Peters(2009)Kober and Peters]{kober2009policy}
Jens Kober and Jan~R Peters.
\newblock Policy search for motor primitives in robotics.
\newblock In \emph{Advances in neural information processing systems}, pp.\
  849--856, 2009.

\bibitem[Koller \& Friedman(2009)Koller and Friedman]{koller2009probabilistic}
Daphne Koller and Nir Friedman.
\newblock \emph{Probabilistic graphical models: principles and techniques}.
\newblock MIT press, 2009.

\bibitem[Levine(2018)]{levine2018rlasinf}
Sergey Levine.
\newblock Reinforcement learning and control as probabilistic inference:
  {T}utorial and review.
\newblock \emph{arXiv preprint arXiv:1805.00909}, 2018.

\bibitem[Mnih et~al.(2013)Mnih, Kavukcuoglu, Silver, Graves, Antonoglou,
  Wierstra, and Riedmiller]{mnih-atari-2013}
Volodymyr Mnih, Koray Kavukcuoglu, David Silver, Alex Graves, Ioannis
  Antonoglou, Daan Wierstra, and Martin Riedmiller.
\newblock Playing atari with deep reinforcement learning.
\newblock In \emph{NIPS Deep Learning Workshop}. 2013.

\bibitem[Mnih et~al.(2016)Mnih, Badia, Mirza, Graves, Lillicrap, Harley,
  Silver, and Kavukcuoglu]{mnih2016asynchronous}
Volodymyr Mnih, Adria~Puigdomenech Badia, Mehdi Mirza, Alex Graves, Timothy
  Lillicrap, Tim Harley, David Silver, and Koray Kavukcuoglu.
\newblock Asynchronous methods for deep reinforcement learning.
\newblock In \emph{Proceedings of the 33rd International Conference on Machine
  Learning ({ICML})}, pp.\  1928--1937, 2016.

\bibitem[Munos(2014)]{munos2014bandits}
R{\'e}mi Munos.
\newblock From bandits to monte-carlo tree search: The optimistic principle
  applied to optimization and planning.
\newblock \emph{Foundations and Trends{\textregistered} in Machine Learning},
  7\penalty0 (1):\penalty0 1--129, 2014.

\bibitem[O'Donoghue(2018)]{klearning}
Brendan O'Donoghue.
\newblock Variational {B}ayesian reinforcement learning with regret bounds.
\newblock \emph{arXiv preprint arXiv:1807.09647}, 2018.

\bibitem[O'Donoghue et~al.(2017)O'Donoghue, Munos, Kavukcuoglu, and
  Mnih]{o2016pgq}
Brendan O'Donoghue, Remi Munos, Koray Kavukcuoglu, and Volodymyr Mnih.
\newblock Combining policy gradient and {Q}-learning.
\newblock In \emph{International Conference on Learning Representations
  (ICLR)}, 2017.

\bibitem[O'Donoghue et~al.(2018)O'Donoghue, Osband, Munos, and
  Mnih]{o2017uncertainty}
Brendan O'Donoghue, Ian Osband, Remi Munos, and Volodymyr Mnih.
\newblock The uncertainty {B}ellman equation and exploration.
\newblock In \emph{Proceedings of the 35th International Conference on Machine
  Learning ({ICML})}, 2018.

\bibitem[Osband \& Van~Roy(2016)Osband and Van~Roy]{osband2016lower}
Ian Osband and Benjamin Van~Roy.
\newblock On lower bounds for regret in reinforcement learning.
\newblock \emph{arXiv preprint arXiv:1608.02732}, 2016.

\bibitem[Osband \& Van~Roy(2017)Osband and Van~Roy]{osband2016posterior}
Ian Osband and Benjamin Van~Roy.
\newblock Why is posterior sampling better than optimism for reinforcement
  learning.
\newblock In \emph{Proceedings of the 34th International Conference on Machine
  Learning ({ICML})}, 2017.

\bibitem[Osband et~al.(2014)Osband, Van~Roy, and Wen]{osband2014generalization}
Ian Osband, Benjamin Van~Roy, and Zheng Wen.
\newblock Generalization and exploration via randomized value functions.
\newblock \emph{arXiv preprint arXiv:1402.0635}, 2014.

\bibitem[Osband et~al.(2016)Osband, Blundell, Pritzel, and
  Van~Roy]{osband2016deep}
Ian Osband, Charles Blundell, Alexander Pritzel, and Benjamin Van~Roy.
\newblock Deep exploration via bootstrapped {DQN}.
\newblock In \emph{Advances In Neural Information Processing Systems}, pp.\
  4026--4034, 2016.

\bibitem[Osband et~al.(2017)Osband, Russo, Wen, and Van~Roy]{osband2017deep}
Ian Osband, Daniel Russo, Zheng Wen, and Benjamin Van~Roy.
\newblock Deep exploration via randomized value functions.
\newblock \emph{arXiv preprint arXiv:1703.07608}, 2017.

\bibitem[Osband et~al.(2018)Osband, Aslanides, and
  Cassirer]{osband2018randomized}
Ian Osband, John Aslanides, and Albin Cassirer.
\newblock Randomized prior functions for deep reinforcement learning.
\newblock In \emph{Advances in Neural Information Processing Systems}, pp.\
  8617--8629, 2018.

\bibitem[Osband et~al.(2019)Osband, Doron, Hessel, Aslanides, , Sezener,
  Saraiva, McKinney, Lattimore, Szepezvari, Singh, Van~Roy, Sutton, Silver, and
  Van~Hasselt]{osband2019bsuite}
Ian Osband, Yotam Doron, Matteo Hessel, John Aslanides, , Eren Sezener, Andre
  Saraiva, Katrina McKinney, Tor Lattimore, Csaba Szepezvari, Satinder Singh,
  Benjamin Van~Roy, Richard Sutton, David Silver, and Hado Van~Hasselt.
\newblock Behaviour suite for reinforcement learning.
\newblock 2019.

\bibitem[Peters et~al.(2010)Peters, M{\"u}lling, and Altun]{peters2010relative}
Jan Peters, Katharina M{\"u}lling, and Yasemin Altun.
\newblock Relative entropy policy search.
\newblock In \emph{AAAI}. Atlanta, 2010.

\bibitem[Rawlik et~al.(2013)Rawlik, Toussaint, and
  Vijayakumar]{rawlik2013stochastic}
Konrad Rawlik, Marc Toussaint, and Sethu Vijayakumar.
\newblock On stochastic optimal control and reinforcement learning by
  approximate inference.
\newblock In \emph{Twenty-Third International Joint Conference on Artificial
  Intelligence}, 2013.

\bibitem[Russo \& Van~Roy(2014)Russo and Van~Roy]{russo2014ids}
Daniel Russo and Benjamin Van~Roy.
\newblock Learning to optimize via information-directed sampling.
\newblock In \emph{Advances in Neural Information Processing Systems}, pp.\
  1583--1591, 2014.

\bibitem[Russo et~al.(2018)Russo, Van~Roy, Kazerouni, Osband, Wen,
  et~al.]{russo2018tutorial}
Daniel~J Russo, Benjamin Van~Roy, Abbas Kazerouni, Ian Osband, Zheng Wen,
  et~al.
\newblock A tutorial on thompson sampling.
\newblock \emph{Foundations and Trends{\textregistered} in Machine Learning},
  11\penalty0 (1):\penalty0 1--96, 2018.

\bibitem[Silver et~al.(2016)Silver, Huang, Maddison, Guez, Sifre, Van
  Den~Driessche, Schrittwieser, Antonoglou, Panneershelvam, Lanctot, Dieleman,
  Grewe, Nham, Kalchbrenner, Sutskever, Lillicrap, Leach, Kavukcuoglu, Graepel,
  and Hassabis]{silver2016mastering}
David Silver, Aja Huang, Chris~J Maddison, Arthur Guez, Laurent Sifre, George
  Van Den~Driessche, Julian Schrittwieser, Ioannis Antonoglou, Veda
  Panneershelvam, Marc Lanctot, Sander Dieleman, Dominik Grewe, John Nham, Nal
  Kalchbrenner, Ilya Sutskever, Timothy Lillicrap, Madeleine Leach, Koray
  Kavukcuoglu, Thore Graepel, and Demis Hassabis.
\newblock Mastering the game of go with deep neural networks and tree search.
\newblock \emph{Nature}, 529\penalty0 (7587):\penalty0 484--489, 2016.

\bibitem[Strehl et~al.(2006)Strehl, Li, Wiewiora, Langford, and
  Littman]{strehl2006pac}
Alexander~L Strehl, Lihong Li, Eric Wiewiora, John Langford, and Michael~L
  Littman.
\newblock {PAC} model-free reinforcement learning.
\newblock In \emph{Proceedings of the 23rd international conference on Machine
  learning}, pp.\  881--888. ACM, 2006.

\bibitem[Sutton \& Barto(2018)Sutton and Barto]{sutton2018reinforcement}
Richard~S Sutton and Andrew~G Barto.
\newblock \emph{Reinforcement learning: An introduction}.
\newblock MIT press, 2018.

\bibitem[Thompson(1933)]{thompson1933likelihood}
William~R Thompson.
\newblock On the likelihood that one unknown probability exceeds another in
  view of the evidence of two samples.
\newblock \emph{Biometrika}, 25\penalty0 (3/4):\penalty0 285--294, 1933.

\bibitem[Todorov(2007)]{todorov2007linearly}
Emanuel Todorov.
\newblock Linearly-solvable markov decision problems.
\newblock In \emph{Advances in neural information processing systems}, pp.\
  1369--1376, 2007.

\bibitem[Todorov(2008)]{todorov2008general}
Emanuel Todorov.
\newblock General duality between optimal control and estimation.
\newblock In \emph{2008 47th IEEE Conference on Decision and Control}, pp.\
  4286--4292. IEEE, 2008.

\bibitem[Todorov(2009)]{todorov2009efficient}
Emanuel Todorov.
\newblock Efficient computation of optimal actions.
\newblock \emph{Proceedings of the national academy of sciences}, 106\penalty0
  (28):\penalty0 11478--11483, 2009.

\bibitem[Toussaint(2009)]{toussaint2009robot}
Marc Toussaint.
\newblock Robot trajectory optimization using approximate inference.
\newblock In \emph{Proceedings of the 26th annual international conference on
  machine learning}, pp.\  1049--1056. ACM, 2009.

\bibitem[Toussaint \& Storkey(2006)Toussaint and
  Storkey]{toussaint2006probabilistic}
Marc Toussaint and Amos Storkey.
\newblock Probabilistic inference for solving discrete and continuous state
  markov decision processes.
\newblock In \emph{Proceedings of the 23rd international conference on Machine
  learning}, pp.\  945--952. ACM, 2006.

\bibitem[Wald(1950)]{wald1950statistical}
Abraham Wald.
\newblock Statistical decision functions.
\newblock 1950.

\bibitem[Watkins(1989)]{watkins1989learning}
Christopher John Cornish~Hellaby Watkins.
\newblock \emph{Learning from delayed rewards}.
\newblock PhD thesis, University of Cambridge England, 1989.

\bibitem[Welch et~al.(1995)Welch, Bishop, et~al.]{welch1995introduction}
Greg Welch, Gary Bishop, et~al.
\newblock An introduction to the {K}alman filter.
\newblock 1995.

\bibitem[Ziebart et~al.(2008)Ziebart, Maas, Bagnell, and
  Dey]{ziebart2008maximum}
Brian~D Ziebart, Andrew Maas, J~Andrew Bagnell, and Anind~K Dey.
\newblock Maximum entropy inverse reinforcement learning.
\newblock 2008.

\end{thebibliography}

\newpage

\section*{Appendix}

\subsection{Soft Q-learning derivation}

We present a derivation of soft Q-learning from the RL as inference
parametric approximation to the probability of optimality. Although
our presentation is slightly different to that of \citet{levine2018rlasinf}
we show here that the resulting algorithms are essentially identical.
Recall from equation (\ref{eq:delta_function}) that the parametric approximation
to optimality we consider is given by
\begin{align*}
\tilde \Prob(\Opt_h(s,a) | \tau_h(s,a) ) &\propto \exp\left(\sum_{(s^\prime,a^\prime) \in \tau_h(s,a)} \beta \Exp^\ell \mu(s^\prime, a^\prime)\right)\\
&= \exp(\beta \Exp^\ell \mu(s,a))\exp\left(\sum_{(s^{\dprime},a^{\dprime}) \in \tau_{h+1}(s^\prime,a^\prime)} \beta \Exp^\ell \mu(s^{\dprime}, a^{\dprime})\right)\\
&= \exp(\beta \Exp^\ell \mu(s,a))\tilde \Prob(\Opt_{h+1}(s^\prime,a^\prime) | \tau_{h+1}(s^\prime,a^\prime))
\end{align*}
where $\tau_h(s,a)$ is a trajectory starting from $(s,a)$ at time $h$ and $\beta
> 0$ is a hyper-parameter. Now we must marginalize out the possible trajectories
$\tau_h$ using the (unknown) system dynamics. Since this is a
certainty-equivalent algorithm we shall use the expected value of the transition
probabilities, under the posterior at episode $\ell$, which means we can write
\[
\tilde \Prob(\tau_h(s,a)) = \Exp^\ell \Pc(s^\prime, s, a) p(a^\prime | s^\prime) \tilde \Prob(\tau_{h+1}(s^\prime, a^\prime)),
\]
and we make the additional assumption that the `prior' $p(a | s)$ is
uniform across all actions $a$ for each $s$ (this assumption is standard in
this framework, see \citet{levine2018rlasinf}).
In this case we obtain
\begin{align*}
\tilde \Prob(\Opt_h(s,a)) &= \sum_{\tau_h(s,a)}\tilde \Prob(\Opt_h(s,a) | \tau_h(s,a)) \tilde \Prob(\tau_h(s,a))\\
&\propto \exp(\beta \Exp^\ell \mu(s,a)) \sum_{s^\prime, a^\prime} \Exp^\ell \Pc(s^\prime, s, a) \sum_{\tau_{h+1}^\prime(s^\prime, a^\prime)}  \tilde \Prob(\Opt_{h+1}(s^\prime, a^\prime) | \tau_{h+1}^\prime(s^\prime, a^\prime))\tilde  \Prob(\tau_{h+1}^\prime(s^\prime, a^\prime)) \\
&=\exp(\beta \Exp^\ell \mu(s,a)) \sum_{s^\prime, a^\prime} \Exp^\ell \Pc(s^\prime, s, a)\tilde  \Prob(\Opt_{h+1}(s^\prime, a^\prime)).
\end{align*}
Now with this we can rewrite
\[
\log \tilde \Prob(\Opt_h(s,a)) =\beta \Exp^\ell \mu(s,a) +  \log \sum_{s^\prime, a^\prime} \Exp^\ell \Pc(s^\prime, s, a) \tilde \Prob(\Opt_{h+1}(s^\prime, a^\prime)) - \log Z(s)
\]
where $Z(s)$ is the normalization constant for state $s$, since $\sum_a \tilde
\Prob(\Opt_h(s,a)) = 1$ for any $s$, and using Jensen's we have the following
bound \[
\log \tilde \Prob(\Opt_h(s,a)) \geq\beta \Exp^\ell \mu(s,a) + \sum_{s^\prime} \Exp^\ell \Pc(s^\prime, s, a) \log \sum_{a^\prime} \tilde \Prob(\Opt_{h+1}(s^\prime, a^\prime)) - \log Z(s)
\]
now if we introduce the soft Q-values that satisfy the soft Bellman equation
\[
\tilde Q_h(s, a) = \Exp^\ell \mu(s,a) + \sum_{s^\prime} \Exp^\ell \Pc(s^\prime, s, a) \beta^{-1} \log \sum_{a^\prime} \exp \beta \tilde Q_{h+1}(s^\prime, a^\prime)
\]
then
\[
\tilde \Prob(\Opt_h(s,a)) \approx \exp \beta \tilde Q_h(s, a) / \sum_b \exp \beta \tilde Q_h(s,b)
\]
and we have the soft Q-learning algorithm (the approximation comes from
the fact we used Jensen's inequality to provide a bound).

\subsection{Proof of theorem \ref{thm-kl}}
\label{app:proof_thm}

\begin{theorem*}
The K-learning value function $V^{\rm K}$ and policy $\pi^{\rm K}$ defined in Table
\ref{t-kl} satisfy the following bound at every state $s \in \Sc$ and $h=0, \ldots H$:
\begin{equation*}
V^{\rm K}_h(s) \geq \Exp V^{M, \star}_h(s) + \beta^{-1} \kld{\Prob(\Opt_h(s))}{\pi_h^{\rm K}(s)}.
\end{equation*}
\end{theorem*}
\begin{proof}
Fix some particular state $s \in \Sc$, and
let the joint posterior over value and optimality be denoted by
\begin{equation}
\label{e-joint}
\Prob(V^{M, \star}_h(s), \Opt_h(s,a)) =  \Prob(Q_h^{M, \star}(s, a) | \Opt_h(s, a)) \Prob(\Opt_h(s,a)),
\end{equation}
where $\Prob(Q_h^{M, \star}(s, a) | \Opt_h(s, a))$ is the conditional
distribution of the Q-values, conditioned on optimality.  Recall that from
equation (\ref{e-kl-approx}) we have
approximated the conditional posterior probability of optimality as
\[
\tilde \Prob( \Opt_h(s, a) | Q^{M, \star}_h(s,a) ) \propto \exp \beta Q^{M, \star}_h(s,a),
\]
for some $\beta > 0$, which when yields
\[
\tilde \Prob(\Opt_h(s, a)) \propto \exp G^Q_h(s,a,\beta).
\]
From Bayes' rule this implies the following approximation to the conditional distribution
\begin{align}
\begin{split}
\label{e-approxp-app}
\tilde \Prob(Q_h^{M, \star}(s, a) | \Opt_h(s, a)) &= \frac{\tilde \Prob( \Opt_h(s, a) | Q^{M, \star}_h(s,a) ) \Prob(Q_h^{M, \star}(s, a))}{\tilde \Prob(\Opt_h(s, a))} \\
&= \Prob(Q_h^{M, \star}(s, a))\exp (\beta Q_h^{M, \star}(s, a) - G^Q_h(s,a,\beta)).
\end{split}
\end{align}
This is known as the \emph{exponential tilt} of the posterior distribution $\Prob(Q_h^{M, \star}(s, a))$
and has a myriad of applications in statistics \citep{asmussen2007stochastic}.
From this we could derive an approximation to the joint posterior
(\ref{e-joint}), however, the K-learning policy does not follow
(\ref{e-pol-cgf}) since computing the cumulant generating function is
non-trivial. Instead we compute the K-values, which are the solution to a
Bellman equation that provide a guaranteed upper bound on the cumulant
generating function, and the K-learning policy is thus
\[
\pi_h^{\rm K}(s,a) \propto \exp(\beta K_h(s,a)),
\]
where we have \citep{klearning}
\begin{equation}
\label{e-app-ub}
\beta K_h(s,a) \geq G^Q_h(s,a)(\beta).
\end{equation}
With that in mind we take our approximation to the joint posterior
(\ref{e-joint}) to be
\[
\tilde \Prob(V^{M, \star}_h(s), \Opt_h(s, a)) = \tilde \Prob(Q_h^{M, \star}(s, a) | \Opt_h(s, a)) \pi^{\rm K}_h(s,a).
\]
Now consider the KL-divergence between the true joint posterior and our approximate
one, a quick calculation yields
\begin{align}
\begin{split}
\label{e-app-kl-expansion}
&\kld{\Prob(V^{M, \star}_h(s), \Opt_h(s, a)) }{\tilde \Prob(V^{M, \star}_h(s), \Opt_h(s, a)) } = \kld{\Prob(\Opt_h(s))}{\pi^{\rm K}_h(s)} +\\
&\quad \sum_{a} \Prob(\Opt_h(s,a)) \kld{\Prob(Q_h^{M, \star}(s, a) | \Opt_h(s, a))}{\tilde \Prob(Q_h^{M, \star}(s, a) | \Opt_h(s, a))},
\end{split}
\end{align}
for timestep $h$ and state $s$.
Considering the terms on the right hand side of (\ref{e-app-kl-expansion}) separately we have
\[
\kld{\Prob(\Opt_h(s))}{\pi^{\rm K}_h(s)} = -\Hc(\Prob(\Opt_h(s))) - \beta \sum_a \Prob(\Opt_h(s, a)) K_h(s,a) + \log \sum_a \exp \beta K_h(s,a)
\]
where $\Hc$ denotes the entropy, and using (\ref{e-approxp-app})
\begin{align*}
&\sum_{a} \Prob(\Opt_h(s,a)) \kld{\Prob(Q_h^{M, \star}(s, a) | \Opt_h(s, a))}{\tilde \Prob(Q_h^{M, \star}(s, a) | \Opt_h(s, a))}\\
&\quad = \sum_a \Prob(\Opt_h(s,a)) G^Q_h(s,a)(\beta) -\beta \sum_a \Prob(\Opt_h(s,a)) \Exp (Q_h^{M, \star}(s, a) | \Opt_h(s, a))\\
&\quad+ \sum_a \Prob(\Opt_h(s,a)) \kld{\Prob(Q_h^{M, \star}(s, a) | \Opt_h(s, a)) }{\Prob(Q_h^{M, \star}(s, a))}.
\end{align*}
Now we sum these two terms, using (\ref{e-app-ub}) and the following identities
\[
\sum_a \Prob(\Opt_h(s,a)) \Exp (Q_h^{M, \star}(s, a) | \Opt_h(s, a)) = \Exp \max_a Q_h^{M, \star}(s, a) = \Exp V^{M, \star}_h(s)
\]
and
\begin{align*}
&\sum_a \Prob(\Opt_h(s,a)) \kld{ \Prob(Q_h^{M, \star}(s, a) | \Opt_h(s, a)) }{\Prob(Q_h^{M, \star}(s, a))} \\
&\quad = \sum_a \Prob(\Opt_h(s,a)) \int \Prob(Q_h^{M, \star}(s, a) | \Opt_h(s, a)) \log(\Prob(\Opt_h(s, a) | Q_h^{M, \star}(s, a))) + \Hc(\Prob(\Opt_h(s)))\\
&\quad \leq \Hc(\Prob(\Opt_h(s))),
\end{align*}
since $\log(\Prob(\Opt_h(s, a) | Q_h^{M, \star}(s, a))) \leq 0$,
we obtain
\[
\kld{\Prob(V^{M, \star}_h(s), \Opt_h(s, a)) }{\tilde \Prob(V^{M, \star}_h(s), \Opt_h(s, a))} \leq \log \sum_a \exp \beta K_h(s,a) - \beta \Exp V_h^{M, \star}(s).
\]
The theorem follows from this and the fact that the K-learning value function is
defined as
\[
V^{\rm K}_h(s) = \beta^{-1} \log \sum_a \exp \beta K_h(s,a)
\]
as well as the fact that
\[
\kld{\Prob(\Opt_h(s))}{\pi^{\rm K}_h(s)} \leq \kld{\Prob(V^{M, \star}_h(s), \Opt_h(s, a)) }{\tilde \Prob(V^{M, \star}_h(s), \Opt_h(s, a))}
\]
from equation (\ref{e-app-kl-expansion}).
\end{proof}

\subsection{Problem \ref{problem:one_unknown} K-learning details}
For a bandit problem the K-learning policy is given by
\[
\pi^{\rm K}_i \propto \exp G^\mu_i(\beta),
\]
which requires the cumulant generating function of the posterior over each arm.
For arm $1$ and the distractor arms there is no uncertainty, in which case the
cumulant generating function is given by
\[
G^\mu_i(\beta) = \mu_i \beta,\quad i = 1, 3, \ldots N.
\]
In the case of arm $2$ the cumulant generating function is
\[
G^\mu_2(\beta) = \log\big((1/2)\exp (2 \beta) + (1/2) \exp (-2 \beta)\big).
\]
In \citep{klearning} it was shown that the optimal choice of $\beta$ is given by
\[
\beta^{\star} = \argmin_{\beta \geq 0}\left( \beta^{-1} \log \sum_{i=1}^N \exp G^\mu_i(\beta)\right),
\]
which requires solving a convex optimization problem in variable $\beta^{-1}$.
In the case of problem \ref{problem:one_unknown} the optimal choice of $\beta
\approx 10.23$, which yields $\pi^{kl}_2 \approx 0.94$.  Then, once arm $2$ has
been pulled once and the true reward of arm $2$ has been revealed, its cumulant
generating function has the same form as the others, and then the optimal choice
of $\beta$ is simply
\[
\beta^{\star} = \argmin_{\beta \geq 0}\left( \beta^{-1} \log \sum_{i=1}^N \exp \mu_i \beta \right)= \infty,
\]
at which point K-learning is greedy with respect to the optimal arm.

\subsection{Implementation details for \texttt{bsuite} evaluation}
\label{app:bsuite-implementation}

All three algorithms use the same neural network architecture consisting of an
MLP (multilayer perceptron) with a single hidden layer with $50$ hidden units.
All three algorithms used a replay buffer of the most recent $10^4$ transitions
to allow re-use of data.  For all three the Adam optimizer
\citep{kingma2014adam} was used with learning rate $10^{-3}$ and batch-size
$128$, and learning is performed at every time-step. For both K-learning and
soft Q-learning the temperature was set at $\beta^{-1} = 0.01$. For Bootstrap
DQN we chose an ensemble of size $20$, and used the randomized prior functions
\citep{osband2018randomized} with scale $3.$. For K-learning, in order to
estimate the cumulant generating function of the reward, we used an ensemble of
neural networks predicting the reward for each state and action and used these
to calculate the empirical cumulant generating function over them. Each of these
was a single hidden layer MLP with $10$ hidden units. Finally, we noted that
actually training a small ensemble of K-networks performed better than a single
network, we used an ensemble of size $10$ for this purpose as well as using
randomized priors to encourage diversity between the elements of the ensemble
with scale $1.0$. The K-learning policy was the Boltzmann policy over all the
ensemble K-values at each state.

\newpage

\onecolumn
\ifx\newgeometry\undefined
\else
\newgeometry{top=20mm, bottom=20mm, left=20mm, right=20mm} 
\fi

\bsuitetitle{Making sense of RL and Inference}
\label{app:bsuite-report}
\bsuiteabstract

\subsection{Agent definition}
\label{app:bsuite-agents}
All agents were run with the same network architecture (a single layer MLP with 50 hidden units a ReLU activation) adapting DQN \citep{mnih-atari-2013}.
Full hyperparameters in Appendix \ref{app:bsuite-implementation}.
\begin{itemize}[noitemsep, nolistsep, leftmargin=*]
    \item {\bf boot\_dqn}: bootstrapped DQN with prior networks \citep{osband2016deep,osband2018randomized}.
    \item {\bf k\_learn}: K-learning via ensemble with prior networks \citep{klearning,osband2018randomized}.
    \item {\bf soft\_q}: soft Q-learning with temperature $\beta^{-1}=0.01$ \citep{o2016pgq}.
\end{itemize}

\subsection{Summary scores}
\label{app:bsuite-scores}

Each \bsuite\ experiment outputs a summary score in [0,1].
We aggregate these scores by according to key experiment type, according to the standard analysis notebook.
A detailed analysis of each of these experiments may be found in a notebook hosted on Colaboratory: \href{http://bit.ly/rl-inference-bsuite}{\texttt{bit.ly/rl-inference-bsuite}}.

\begin{figure}[h!]
\centering
\begin{minipage}[t]{.5\textwidth}
  \centering
  \includegraphics[width=\textwidth,height=70mm,keepaspectratio]{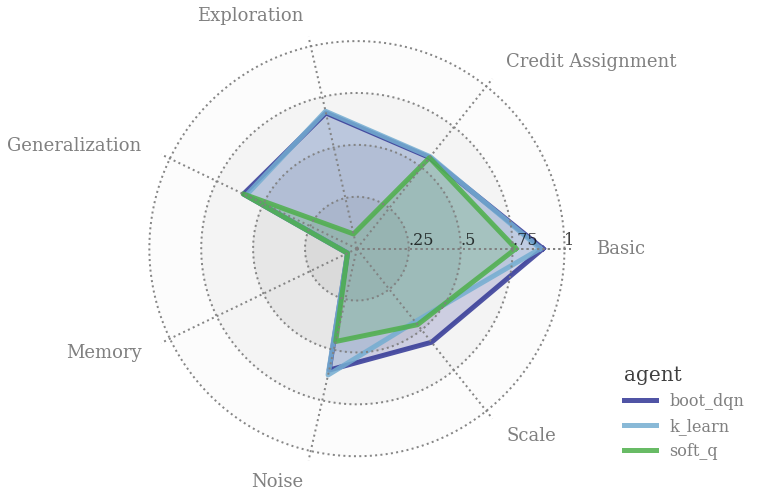}
  \captionof{figure}{Snapshot of agent behaviour.}
  \label{fig:radar}
\end{minipage}%
\begin{minipage}[t]{.5\textwidth}
  \centering
  \includegraphics[width=\textwidth,height=70mm,keepaspectratio]{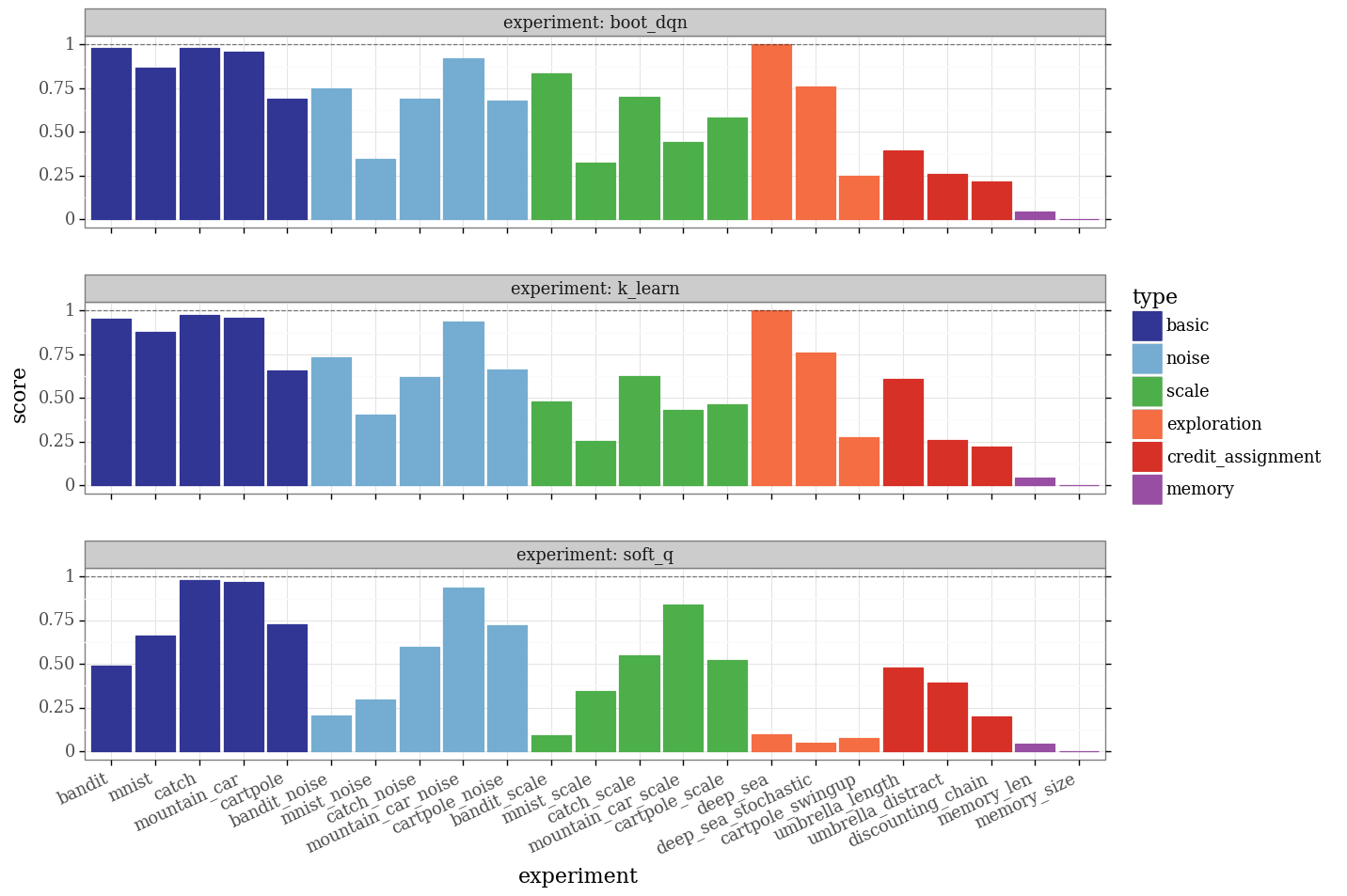}
  \captionof{figure}{Score for each \bsuite\ experiment.}
  \label{fig:bar}
\end{minipage}
\end{figure}

\subsection{Results commentary}
\label{app:bsuite-commentary}

Overall, we see that the algorithms K-learning and Bootstrapped DQN perform extremely similarly across \bsuite\ evaluations.
However, there is a clear signal that soft Q-learning performs markedly worse on the tasks requiring efficient exploration.
This observation is consistent with the hypothesis that algorithms motivated by `RL as Inference' fail to account for the value of exploratory actions.

Beyond this major difference in exploration score, we see that Bootstrapped DQN outperforms the other algorithms on problems varying `Scale'.
This too is not surprising, since both soft Q and K-learning rely on a temperature tuning that will be problem-scale dependent.
Finally, we note that soft Q also performs worse on some `basic' tasks, notably `bandit' and `mnist'.
We believe that the relatively high temperature (tuned for best performance on Deep Sea) leads to poor performance on these tasks with larger action spaces, due to too many random actions.

\newpage

\end{document}